\pdfoutput=1
\documentclass[10pt,journal,cspaper,compsoc]{IEEEtran}

\usepackage{cite}

\usepackage{times}
\usepackage{xcolor}
\usepackage{soul}
\usepackage[utf8]{inputenc}
\usepackage{epsfig}
\usepackage{times}
\usepackage{helvet}
\usepackage{courier}
\usepackage{amsthm}
\usepackage{amssymb,amsmath}
\usepackage{dsfont}
\usepackage{graphicx}
\usepackage{url}
\usepackage{bm}
\usepackage{algorithm,algorithmic}
\theoremstyle{definition}

\newtheorem{cor}{Corollary}
\newtheorem{defn}{Definition}

\newtheorem{pro}{Proposition}

\newtheorem{lemma}{Lemma}
\DeclareMathOperator{\Lip}{Lip}

\DeclareMathOperator{\diam}{diam}
\DeclareMathOperator{\ddim}{ddim}

\DeclareMathOperator{\sign}{sign}

\DeclareMathOperator{\sumkmin}{sumKmin}
\usepackage{tabularx}

\begin{document}

\title{Learning Local Metrics and Influential Regions for Classification}
%
%
%

\author{Mingzhi~Dong, Yujiang~Wang, Xiaochen~Yang, Jing-Hao~Xue
\thanks{M.~Dong, X.~Yang and J.-H.~Xue are with the Department of Statistical Science, University College London, London WC1E 6BT, UK (e-mail: mingzhi.dong.13@ucl.ac.uk; xiaochen.yang.16@ucl.ac.uk; jinghao.xue@ucl.ac.uk).}
\thanks{Y.~Wang is with the Department of Computing, Imperial College London, London SW7 2AZ, UK (e-mail: yujiang.wang14@imperial.ac.uk).}}


\maketitle

\begin{abstract}
The performance of distance-based classifiers heavily depends on the underlying distance metric, so it is valuable to learn a suitable metric from the data. To address the problem of multimodality, it is desirable to learn local metrics. In this short paper, we define a new intuitive distance with local metrics and influential regions, and subsequently propose a novel local metric learning method for distance-based classification. Our key intuition is to partition the metric space into influential regions and a background region, and then regulate the effectiveness of each local metric to be within the related influential regions. We learn local metrics and influential regions to reduce the empirical hinge loss, and regularize the parameters on the basis of a resultant learning bound. Encouraging experimental results are obtained from various public and popular data sets.
\end{abstract}


\section{Introduction}

Classification is a fundamental task in the field of machine learning. While deep learning classifiers have obtained superior performance on numerous applications, they generally require a large amount of labeled data. For small data sets, traditional classification algorithms remain valuable. 

The nearest neighbor (NN) classifier is one of the oldest established methods for classification, which compares the distances between a new instance and the training instances. 
However, with different metrics, the performance of NN would be quite different. 
Hence it is very beneficial if we can find a well-suited and adaptive distance metric for specific applications. 
To this end, metric learning is an appealing technique. 
It enables the algorithms to automatically learn a metric from the available data. 
Metric learning with a convex objective function was first proposed in the seminal work of Xing et al.~\cite{xing2002distance}. 
After that, many other metric learning methods have been developed and widely adopted, such as the Large Marin Nearest Neighbor (LMNN)~\cite{weinberger2009distance} and the Information Theoretic Metric Learning~\cite{davis2007information}. 
Some theoretical work has also been proposed for metric learning, especially on deriving different generalization bounds~\cite{jin2009regularized,guo2014guaranteed,cao2016generalization,verma2015sample} and deep networks have been used to represent nonlinear metrics ~\cite{hu2014discriminative,lu2015multi}. 
In addition, metric learning methods have bee developed for specific purposes, including 
multi-output tasks~\cite{liu2018metric}, multi-view learning~\cite{hu2017sharable}, medical image retrieval~\cite{yang2010boosting}, kinship verification tasks~\cite{lu2014neighborhood,yan2014discriminative}, face recognition tasks ~\cite{huang2017cross}, tracking problems~\cite{wang2017tracklet} and so on.

Most aforementioned methods use a single metric for the whole metric space and thus may not be well-suited for data sets with multimodality. 
To solve this problem, local metric learning algorithms have been proposed~\cite{frome2007learning,weinberger2009distance,wang2012parametric,huang2013reduced,bohne2014large,shi2014sparse,saxena2015coordinated,st2017sparse, noh2018generative}. 

Most of these localized algorithms can be categorized into two groups: 1) Each data point or cluster of data points has a local metric $M(\bm x_i)$. 
This, however, results in an asymmetric distance as illustrated in~\cite{wang2012parametric}, i.e.~$M(\bm x_i)\neq M(\bm x_j)$ would cause $D(\bm x_i, \bm x_j; M(\bm x_i))\neq D(\bm x_j, \bm x_i; M(\bm x_j))$. 2) Each line segment or cluster of line segments has a local metric, i.e.~$M(\bm x_i, \bm x_j)$. 
The definitions of $M(\bm x_i, \bm x_j)$, such as  $\sum_{k} w_k (\bm x_i, \bm x_j) M_k $ in \cite{bohne2014large} where $w_k$ is defined as $P(k|\bm x_i)+ P(k|\bm x_j)$ to guarantee the symmetry and  $P(k|\bm x_i)$ or $P(k|\bm x_j)$ is based on the posterior probability that the point $\bm x$ belongs to the $k$th Gaussian cluster in a Gaussian mixture (GMM), are nonetheless not very intuitive.

\begin{figure}
\centering
\includegraphics[width=0.4\textwidth]{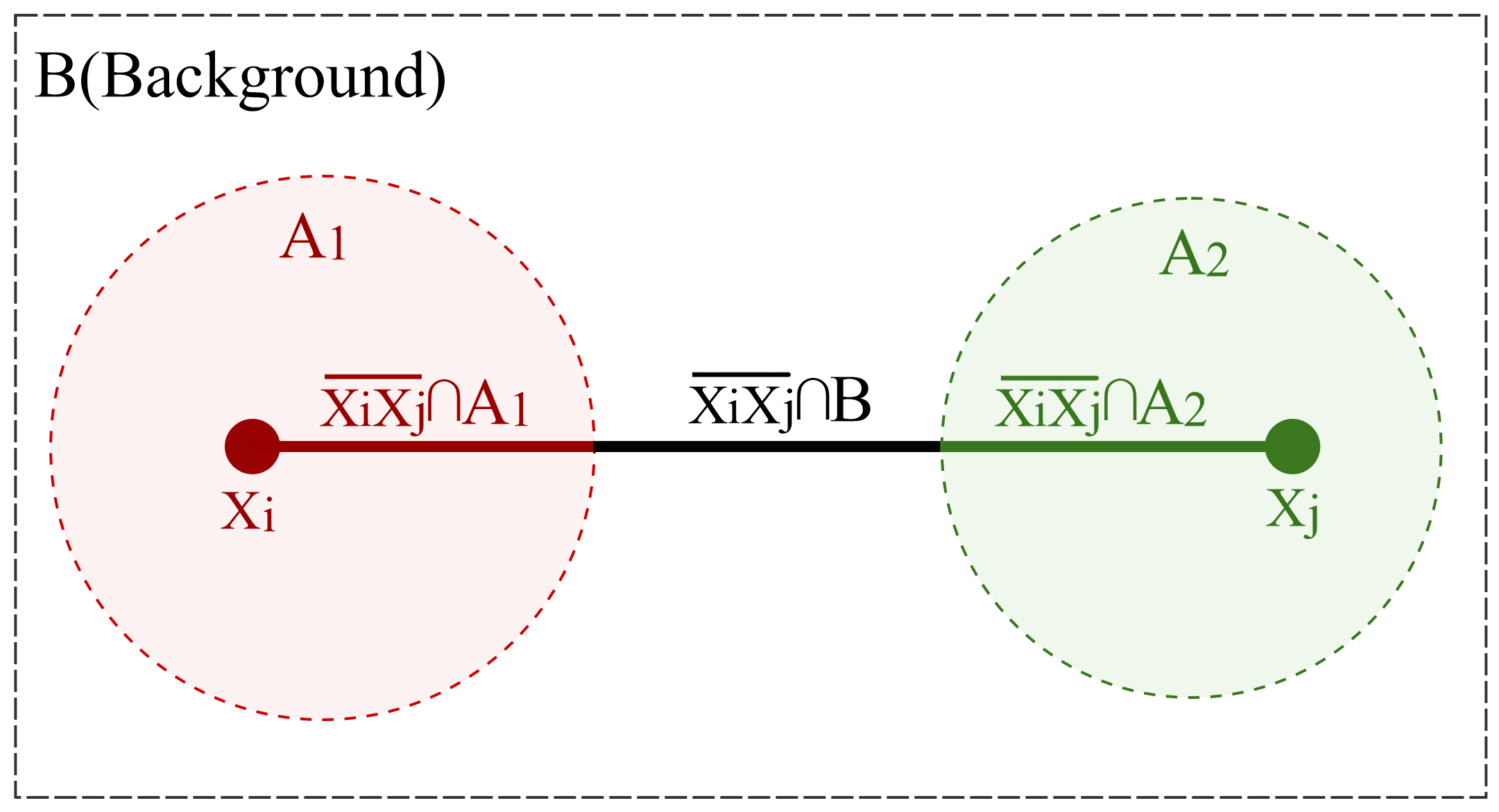}
\caption{An example of calculating the distance between two points $\bm x_i$ and $\bm x_j$. $ A_1$ and $A_2$ are different influential regions with metrics $M(A_1)$ and $M( A_2)$, and B is the background region with metric $M(B)$. The distance between $\bm x_i$ and $\bm x_j$ equals to the sum of three line segments' local distances, i.e. $l(\overline{\bm x_i \bm x_j} \cap A_1;M(A_1))$, $l(\overline{\bm x_i \bm x_j} \cap A_2;M(A_2))$ and $l(\overline{\bm x_i \bm x_j} \cap B;M(B))$.}\label{fig_distance_intuition}
\end{figure}

In this short paper, we define an intuitive, new symmetric distance, and a novel local metric learning method.
By splitting the metric space into influential regions and a background region, we define the distance between any two points as the sum of lengths of line segments in each region, as illustrated in Figure~\ref{fig_distance_intuition}. 
Building multiple influential regions solves the multimodality issues; and learning a suitable local metric in each influential region improves class separability, as shown in Figure~\ref{fig_local_area}.


\begin{figure}
\centering
\includegraphics[width=0.4\textwidth]{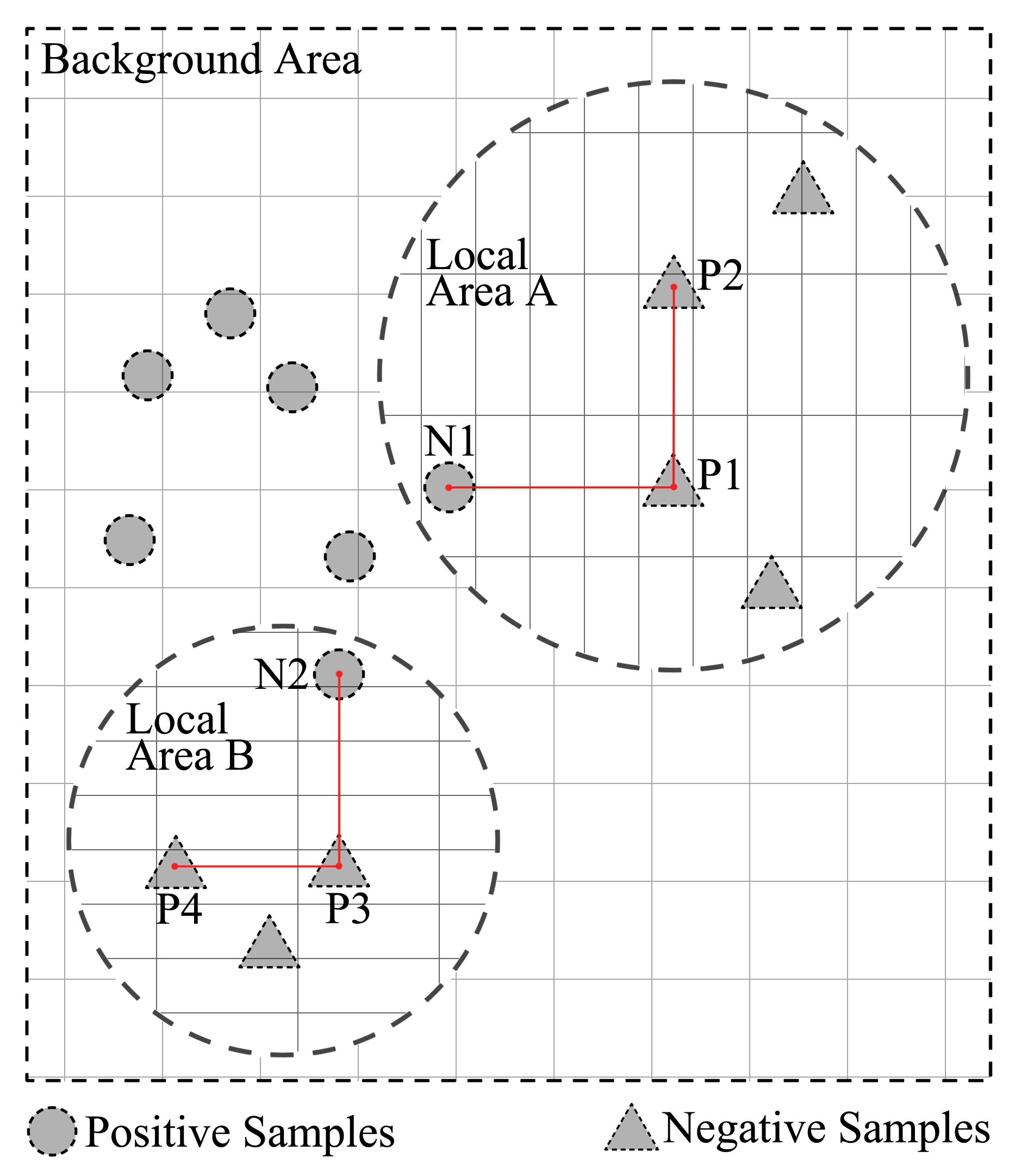}
\caption{An illustration of learning local influential areas. The distance between the adjacent vertical/horizontal grids is one unit. The location and radius of a local area could be learned and a suitable local metric could help to enhance the separability of the data, such as increasing $l(\overline{N_1 P_1})$ and $l(\overline{N_2 P_3})$ while decreasing $l(\overline{P_1 P_2})$ and $l(\overline{P_3 P_4})$.}\label{fig_local_area}
\end{figure}

To establish our new distance and local metric learning method, we first define some key concepts, namely influential regions, local metrics and line segments, which lead to the definition of the new distance. 
Then we calculate the distance by discussing the geometric relationship between line segment and influential regions. After that, we use the proposed local metric to build a novel classifier and study its learnablity. 
The penalty terms from the derived learning bound, together with the empirical hinge loss, form an optimization problem, which is solved via gradient descent due to the non-convexity. 
Finally we experiment the proposed local metric learning algorithm on 14 publicly available data sets. On eight of these data sets, the proposed algorithm achieves the best performance, much better than the state-of-the-art metric learning competitors.

\section{Definitions of Influential Regions, Local Metrics and Distance}

In this section, we will first define influential regions $A_s, s=1,\dots, S$, and the background region $B$. 
With a local metric for each region $M(A_s)$ and $M(B)$, the distance between $\bm x_i$ and $\bm x_j$ will be defined as the sum of lengths of line segments in each influential region and the background region, as illustrated in Figure~\ref{fig_distance_intuition}. 
Since the metric is defined with respect to line segments, the distance is symmetric, i.e. $D(\bm x_i, \bm x_j)=D_{M(\overline{\bm x_i \bm x_j})}(\bm x_i, \bm x_j)=D_{M(\overline{\bm x_j \bm x_i})}(\bm x_j, \bm x_i)=D(\bm x_j, \bm x_i)$.

To simplify the calculation required later, we restrict the shape of each influential region to be a ball. 

\begin{defn}
\textit{Influential regions} are defined to be any set of $n$-balls inside the metric space: 
\begin{equation*}
A = \{ A_s, s=1,\dots, S\},
\end{equation*}
where $S$ denotes the number of influential regions; $A_s= Ball(\bm o_s, r_s)$, in which $Ball(\bm o_s, r_s)$ denotes a ball with the center at $\bm o_s$ and radius of $r_s$; the location of each influential region is determined by the Euclidean distance; and points $\bm x\in A_s$ construct a set with the following form
\begin{equation}\label{defn_influentialregion}
    \{\bm x|(\bm o_s - \bm x)^T (\bm o_s - \bm x) \leq r_s^2\}. 
\end{equation}
\end{defn}
 
\begin{defn}
\textit{Background region} is defined to be the region excluding influential regions:
\begin{equation*}
    B=U- \bigcup_{s=1,\dots,S} A_s,
\end{equation*}
where $U$ indicates the universe set.
\end{defn}


Throughout this paper, the distance between two points $\bm x_i$ and $\bm x_j$ is equivalent to the length of line segment $\overline{\bm x_i \bm x_j}$, i.e. $D(\bm x_i, \bm x_j) = l(\overline{\bm x_i \bm x_j})$. Length $l(\overline{\bm x_i \bm x_j})$ in influential regions and the background region will be defined separately with respective metrics. 

\begin{defn}
Each influential region $A_s$ has its own \textit{local metric} $M(A_s)$. The length of a line segment $\overline{\bm x_i \bm x_j}$ inside an influential region $A_s$ is defined as\footnote{Since influential regions are restricted to be ball-shaped and a ball is a convex set, the line segment  $\overline{\bm x_i \bm x_j}$ would lie in the ball for any two point $\bm x_i$ and $\bm x_j$ inside the ball. } 
\begin{equation}
\begin{split}
  l(\overline{\bm x_i \bm x_j}; M(A_s))  =& D_{M(A_s)}(\bm x_i, \bm x_j) \\
    = & \sqrt{(\bm x_i -\bm x_j)^T M(A_s)(\bm x_i -\bm x_j)  }.
\end{split}
\end{equation}
To make illustrations more intuitive, the distance adopted in this paper will be based on the Mahalanobis distance\footnote{This is different the usually adopted squared Mahalanobis distance and enjoys convenience when solving the optimization problem.}. 
\end{defn}

\begin{defn}
The background region $B$ has a \textit{background metric} $M(B)$. For any two points $\bm x_i, \bm  x_j \in B$ and  $\overline{\bm x_i \bm x_j} \subset B$, the length of a line segment is defined as
\begin{eqnarray*}
    l(\overline{\bm x_i \bm x_j}; M(B)) & =& D_{M(B)}(\bm x_i, \bm x_j) \\
    &= & \sqrt{(\bm x_i -\bm x_j)^T M(B)(\bm x_i -\bm x_j)  }.
\end{eqnarray*}
\end{defn}
We make two remarks here:
\begin{enumerate}
    \item While the metrics $M(A_s)$ and $M(B)$ will be learned inside influential regions and the background region, the Euclidean distance is used to determine the location of influential regions.
    \item For $\bm x_i,\bm x_j \in B$ and $\overline{\bm x_i \bm x_j} \not\subset B$, the distance between $\bm x_i$ and $\bm x_j$ is generally different from $D_{M(B)}(\bm x_i, \bm x_j)$. It is because some parts of the line segment $\overline{\bm x_i \bm x_j}$ may lie in influential regions so their lengths should be calculated via the related local metrics.
\end{enumerate}

To calculate the distance between any $\bm x_i\in U$ and $\bm x_j\in U$, we need to consider the relationship between the line segment $\overline{\bm x_i \bm x_j}$ and influential regions, which can be simplified as one of the following three cases: no-intersection, tangent and with-intersection. 

\begin{defn}
The \textit{intersection} of a line segment $\overline{\bm x_i \bm x_j}$ and an influential region $A_s$ is denoted as $A_s \cap \overline{\bm x_i \bm x_j}$. 
In the case of no-intersection,
$A_s\cap \overline{\bm x_i \bm x_j}=\emptyset$; 
in the case of tangent, $A_s \cap \overline{\bm x_i \bm x_j}= \bm t^s_{ij}$, where $\bm t^s_{ij}$ is the tangent point; 
in the case of with-intersection,  $A_s\cap \overline{\bm x_i \bm x_j}= \overline{\bm p^s_{ij} \bm q^s_{ij}}$, 
where $\overline{\bm p^s_{ij} \bm q^s_{ij}}$ is the maximum sub-line segment of $\overline{\bm x_i \bm x_j}$ inside $A_s$, $\bm p^s_{ij}$ is the point which lies closer to $\bm x_i$ and $\bm q^s_{ij}$ is the point which lies closer to $\bm x_j$.
On the other hand, the \textit{intersection} of a line segment $\overline{\bm x_i \bm x_j}$ and the background region B is defined as 
\begin{equation}
\begin{split}
    B \cap \overline{\bm x_i \bm x_j} &=  \overline{\bm x_i \bm x_j} - \bigcup_{s=1\dots S} (A_s\cap \overline{\bm x_i \bm x_j}),
\end{split}
\end{equation}
where $\bigcup_{s=1\dots S}  (A_s \cap \overline{\bm x_i \bm x_j})$ is the union of intersections between the line segment and all influential regions. It could also be understood as a set of non-overlapping line segments\footnote{This could be easily proved by recursively combining any overlapping line segments until no overlapping one is found.}.
\end{defn}

Accordingly, the length of line segment $\overline{\bm x_i \bm x_j}$ can be calculated through the length of intersection. 

\begin{defn}
The \textit{length of intersection} of a line segment $\overline{\bm x_i \bm x_j}$ and an influential region $A_s$ is defined as $l(A_s\cap \overline{\bm x_i \bm x_j}; M(A_s))$.  
In the case of tangent or no-intersection, $l(A_s\cap \overline{\bm x_i \bm x_j}; M(A_s))\triangleq 0 $; in the case of with-intersection, it is defined to be the length of $\overline{\bm p^s_{ij} \bm q^s_{ij}}$, i.e.~$l(A_s \cap \overline{\bm x_i \bm x_j}; M(A_s)) = l(\overline{\bm p^s_{ij} \bm q^s_{ij}}; M(A_s))$.
On the other hand, the \textit{length of the intersection} of a line segment $\overline{\bm x_i \bm x_j}$ and the background region $B$ is defined as 
\begin{equation}
\begin{split}
    l(B\cap \overline{\bm x_i \bm x_j}; M(B))& =  l(\overline{\bm x_i \bm x_j}; M(B)) \\
&-l( \bigcup_{s=1\dots S} (A_s\cap \overline{\bm x_i \bm x_j});M(B)). 
\end{split}
\end{equation}
\end{defn}
\begin{defn} 
The \textit{length of line segment} is defined as
\begin{equation}\label{distance}
\begin{split}
   l(\overline{\bm x_i\bm x_j}; M(\overline{\bm x_i \bm x_j})) 
   &=  \sqrt{(\bm x_i -\bm x_j)^T M(\overline{\bm x_i \bm x_j})(\bm x_i -\bm x_j)  }\\
   &=l(B\cap \overline{\bm x_i \bm x_j}; M(B))\\
   &+\sum_s l(A_s \cap \overline{\bm x_i \bm x_j}; M(A_s)),
\end{split}
\end{equation}
where $M(\overline{\bm x_i \bm x_j})$ is the metric of the line segment $\overline{\bm x_i \bm x_j}$. ${M(\overline{\bm x_i \bm x_j})}$ will be simplified as ${M}$ afterwards.
\end{defn}

\section{Calculation of Distances}

\subsection{Calculation of the length of intersection with influential regions}


\begin{figure}
\centering
\includegraphics[width=0.45\textwidth]{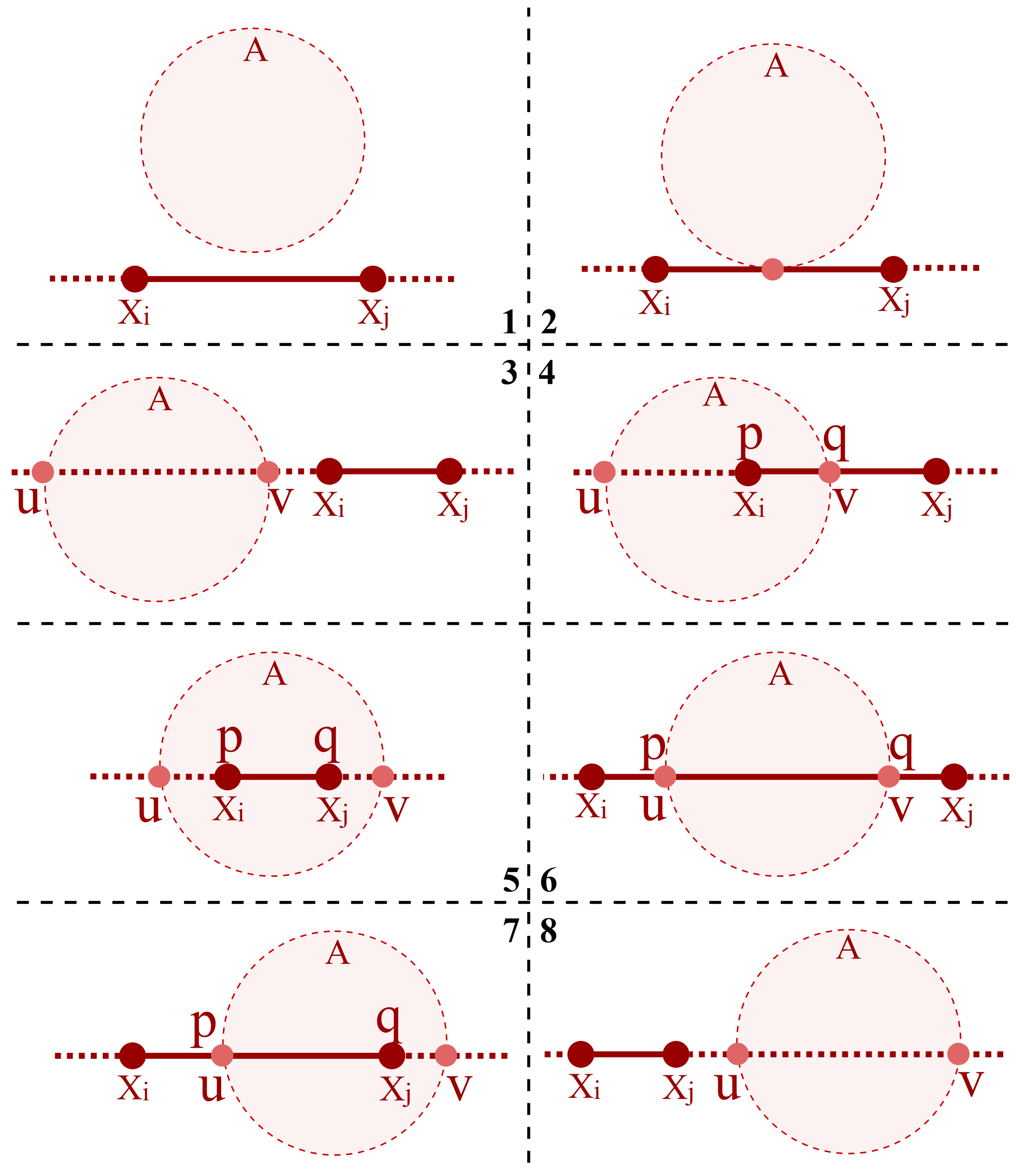}
\caption{ The positions of $\bm u, \bm v$ (intersection points between line $\bm x_i \bm x_j$ and the influential region $A$) and $\bm p, \bm q$ (intersection points between line segment $\overline{\bm x_i \bm x_j}$ and $A$) under different situations.}
\label{fig_uv}
\end{figure}

\begin{table} 
\caption{Summary of notation.} 
\begin{center}
\begin{tabular}{c|c}
\hline
Notation & Detail \\
\hline
$a$ &    $(\bm x_j-\bm x_i)^T (\bm x_j-\bm x_i)$ \\
\hline
$b$ &   2$(\bm x_j - \bm x_i)^T (\bm x_i - \bm o_s)$\\
\hline
$c$ &   $(\bm x_i -\bm o_s)^T  (\bm x_i - \bm o_s)-r_s^2$\\
\hline
$\Delta$ &  $b^2-4ac$\\
\hline
$   \lambda_u$ & $\frac{-b -\sqrt{\Delta} }{2a}$\\
\hline
$\lambda_v$ & $\frac{-b +\sqrt{\Delta} }{2a}$ \\
\hline
$  \bm u$ & $\bm x_i + \lambda_u(\bm x_j-\bm x_i)$\\
\hline
$\bm v$ & $\bm x_i + \lambda_v(\bm x_j-\bm x_i)$\\
\hline
$  \bm p$ & $\bm x_i + \lambda_p(\bm x_j-\bm x_i)$\\
\hline
$\bm q$ & $\bm x_i + \lambda_q(\bm x_j-\bm x_i)$\\
\hline
$\gamma$ & $\lambda_q-\lambda_p$\\
\hline
\end{tabular}
\end{center}
\label{table:notation}
\end{table}

\begin{table*}
\caption{Different cases of $l(A \cap \overline{\bm x_i \bm x_j}; M(A_s))$. The column of `Line $\bm x_i \bm x_j$' indicates the relationship between the line $\bm x_i \bm x_j$ and the influential region, which is determined by the value of $\Delta$; the column of `$\overline{\bm p \bm q}$' indicates the relationship between the line segment $\overline{\bm p \bm q}$ and the influential regions, which is determined by the values of $\lambda_u$ and $\lambda_v$.} 
\label{lambda}
\begin{center}
\begin{tabular}{c c|c|c |c|c|c}
\hline
$\Delta$ & Line $\bm x_i \bm x_j$ & Values of $\lambda_u$ $\lambda_v$ &  $\overline{\bm p \bm q}$ & Values of $\lambda_p$ $\lambda_q$   & $l= l(A_s \cap \overline{\bm x_i \bm x_j}; M(A_s))$ &Illustration  \\
\hline
$\Delta < 0$ & No-intersect&   & & &$l=0$&Figure \ref{fig_uv}.1\\
\hline
$\Delta = 0$ & Tangent&   & & & $l=0$&Figure \ref{fig_uv}.2\\
\hline
&& $\lambda_u<0, \lambda_v < 0$
&$\emptyset$& $\lambda_p,\lambda_q \triangleq 0$
& &Figure \ref{fig_uv}.3\\
&& $\lambda_u<0, 0\leq \lambda_v  \leq1$
&$\overline{\bm x_i \bm v}$& $\lambda_p=0, \lambda_q=\lambda_v$
& &Figure \ref{fig_uv}.4\\
$\Delta > 0$ & with-intersect& $\lambda_u<0, \lambda_v > 1$
&$\overline{\bm x_i \bm x_j}$& $\lambda_p=0, \lambda_q=1$
&\tiny{ $l=\gamma \sqrt{(\bm x_i-\bm x_j)^T M(A_s) (\bm x_i-\bm x_j)}$}&Figure \ref{fig_uv}.5\\
&& $0\leq \lambda_u\leq 1, 0\leq \lambda_v  \leq 1$
&$\overline{\bm u \bm v}$& $\lambda_p=\lambda_u, \lambda_q=\lambda_v $
&$\gamma=\lambda_q-\lambda_p$&Figure \ref{fig_uv}.6\\
&& $0\leq \lambda_u\leq 1,  \lambda_v  > 1$
&$\overline{\bm u \bm x_j}$& $\lambda_p=\lambda_u, \lambda_q=1$
& &Figure \ref{fig_uv}.7\\
&& $\lambda_u>1,  \lambda_v  > 1$
&$\emptyset$& $\lambda_p,\lambda_q \triangleq 1$
& &Figure \ref{fig_uv}.8\\
\hline
\end{tabular}
\end{center}
\end{table*}

We will first provide an intuitive explanation of calculating the length of intersection with influential regions, as illustrated in Figure \ref{fig_uv}. 
If the line $\bm x_i \bm x_j$ does not intersect with or is the tangent to the influential ball, the length is zero. This is equivalent to identifying the start and end points of line $\bm x_i\bm x_j$ and the ball, $\bm u, \bm v$, via one variable quadratic equation.
If the line intersects with the ball, we will calculate the length by considering the relationship between the intersection of the line $\bm x_i \bm x_j$ and the influential ball, i.e.~$\bm u \bm v$, and the intersection of the line segment $\overline{\bm x_i \bm x_j}$ and the influential ball, i.e.~$\bm p \bm q$. $\bm p, \bm q$ can be obtained based on points $\bm u, \bm v$ and the constraint that the start and end points should be on the linear segment $\overline{\bm x_i \bm x_j}$.


\begin{defn}
The \textit{intersection points} of the line $\bm x_i\bm x_j$ and the influential region $A_s$ are represented as $\bm u=\bm x_i + \lambda_u(\bm x_j -\bm x_i)$ and $\bm v=\bm x_i+ \lambda_v(\bm x_j -\bm x_i)$, where $ \lambda_u,\lambda_v \in \mathds{R}$, $\lambda_u \leq \lambda_v$ and $\lambda_u,\lambda_v$ are called the \textit{intersection coefficients} between the line $\bm x_i \bm x_j$ and $A_s$. The \textit{intersection points} of the line segment $\overline{\bm x_i\bm x_j}$ and the influential region are represented as $\bm p=\bm x_i + \lambda_p(\bm x_j -\bm x_i)$ and $\bm q=\bm x_i+ \lambda_q(\bm x_j -\bm x_i)$, where $0\leq \lambda_p \leq \lambda_q \leq 1$ and $\lambda_p,\lambda_q$ are called the \textit{intersection coefficients} between the line segment $\overline{\bm x_i \bm x_j}$ and $A_s$.  $\gamma=\lambda_q-\lambda_p$ is called the \textit{intersection ratio}.
\end{defn}

\begin{pro}
The length of intersection between line segment $\overline{\bm x_i \bm x_j}$ and the influential region $A_s$, with the intersection points $\bm p,\bm q$ and intersection coefficients $\lambda_p,\lambda_q$, is
\begin{equation}
\begin{split}
l(A \cap \overline{\bm x_i \bm x_j}; M(A_s)) &= \sqrt{(\bm q-\bm p)^T M(A_s) (\bm q-\bm p)} \\
&=\gamma \sqrt{(\bm x_i-\bm x_j)^T M(A_s) (\bm x_i-\bm x_j)}.
\end{split}
\end{equation}
\end{pro}
As shown in the above proposition, the length of intersection can be calculated given the local metric $M(A_s)$ and $\gamma$, where the latter term can be obtained from $\lambda_q$ and $\lambda_p$.

Now we discuss the computation of $\gamma$, which can be divided into two steps.

1) Calculate the intersection points of the line $\bm x_i \bm x_j$ and the ball: $\bm u$ and $\bm v$, i.e.~$\bm x_i + \lambda_u(\bm x_j-\bm x_i)$ and $\bm x_i + \lambda_v(\bm x_j-\bm x_i)$. 

The coefficients $\lambda_u$ and $\lambda_v$ could be easily solved through the following quadratic equation with one variable:
\begin{equation}\label{lambda12}
\|\bm x_i + \lambda(\bm x_j - \bm x_i) -\bm o_s\|_2^2 = r_s^2,
\end{equation}
with $\Delta = b^2-4ac =[2(\bm x_j - \bm x_i)^T (\bm x_i - \bm o_s) ]^2 - 4[(\bm x_j-\bm x_i)^T (\bm x_j-\bm x_i)] [(\bm x_i -\bm o_s)^T  (\bm x_i - \bm o_s)-r_s^2]$; and when $\Delta>0$, the solutions $\lambda^s_{u,ij} \leq  \lambda^s_{v,ij}$ to the above equation are
\begin{align*}
\lambda^s_{u,ij} &=\frac{-b -\sqrt{\Delta} }{2a} = \frac{-2(\bm x_j-\bm x_i)^T (\bm x_i - \bm o_s) - \sqrt{\Delta}}{2(\bm x_j-\bm x_i)^T (\bm x_j-\bm x_i)},\\
\lambda^s_{v,ij} &=\frac{-b +\sqrt{\Delta} }{2a} = \frac{-2(\bm x_j-\bm x_i)^T (\bm x_i - \bm o_s) + \sqrt{\Delta}}{2(\bm x_j-\bm x_i)^T (\bm x_j-\bm x_i)}.
\end{align*}
Hence the two intersection points between the ball and the line become
\begin{align*}
\bm u^s_{ij}&= \bm x_i + \lambda^s_{u,ij} (\bm x_j - \bm x_i),\\
\bm v^s_{ij}&= \bm x_i + \lambda^s_{v,ij} (\bm x_j - \bm x_i).
\end{align*} 

For simplicity, the superscript $s$ and subscript $ij$ for $\lambda$, $u$, $v$, $p$ and $q$ will be discarded if no confusion is caused. 

2) Calculate the intersection points of the line segment $\overline{\bm x_i \bm x_j}$ and the ball: $\bm p$ and $\bm q$, i.e.~$\bm x_i + \lambda_p(\bm x_j-\bm x_i)$ and $\bm x_i + \lambda_q(\bm x_j-\bm x_i)$.

We check the number of solutions to (\ref{lambda12}). If (\ref{lambda12}) has 0 or 1 solution, the line has no intersection or is tangent to the region, and thus $l(A \cap \overline{\bm x_i \bm x_j}; M(A_s))=0$. 
If it has two solutions, the intersection between the line and the ball $A_s$ is a line segment $\overline{\bm u \bm v}$. 
Based on the value of $\lambda_u, \lambda_v$\footnote{If and only if the value of $\lambda_u$ or $\lambda_v$ lies in the range of $[0,1]$, the corresponding point lies inside the line segment $\overline{\bm x_i \bm x_j}$.}, we can obtain the relationship between $\overline{\bm u \bm v}$ and $\overline{\bm p \bm q}$ and get the values of $\lambda_p$ and $\lambda_q$ from 
\begin{align*}
    \lambda_p&=\min(\max(\lambda_u, 0),1),\\
    \lambda_q&=\min(\max(\lambda_v, 0),1).
\end{align*}

A summary of the notation used in this section is listed in Table~\ref{table:notation}; the details of the distance calculation are illustrated in Figure~\ref{fig_uv} and Table~\ref{lambda}.

\subsection{Calculation of the length of intersection with local metrics}

\begin{pro}\label{pro:D-nonoverlap}
In the case of non-overlapping influential regions, i.e. $A_i \cap A_j = \emptyset,  \forall i\neq j$,  
\begin{equation}
\begin{split}
D_M(\bm x_i\bm x_j) & \triangleq l(\overline{\bm x_i\bm x_j}; M(\overline{\bm x_i \bm x_j}))\\
& =  \gamma_b \sqrt{(\bm x_i-\bm x_j)^T M(B) (\bm x_i-\bm x_j)} \\
&+ \sum_s \gamma_s \sqrt{(\bm x_i-\bm x_j)^T M(A_s) (\bm x_i-\bm x_j)}\\
& =  (1-\sum_s \gamma_s) \sqrt{(\bm x_i-\bm x_j)^T M(B) (\bm x_i-\bm x_j)} \\
&+ \sum_s \gamma_s \sqrt{(\bm x_i-\bm x_j)^T M(A_s) (\bm x_i-\bm x_j)},
\end{split}
\label{eq:D-nonoverlap}
\end{equation}
where $\gamma_b$ is defined as the intersection ratio of the background region, and in the non-overlapping case $\gamma_b=1-\sum_s \gamma_s$.
\end{pro}

Proposition~\ref{pro:D-nonoverlap} suggests that the distance can be obtained once we have metrics ($M(A_s)$, $M(B)$) and the intersection ratio $\gamma_s$. As all calculations are in closed form, the computation is efficient.  

In the case of overlapping influential regions, we have the same formula as (\ref{eq:D-nonoverlap}):
\begin{equation}\label{equ_distance}
\begin{split}
       D_M(\bm x_i\bm x_j) & \triangleq l(\overline{\bm x_i\bm x_j}; M(\overline{\bm x_i \bm x_j}))\\ 
       & =  \gamma_b \sqrt{(\bm x_i-\bm x_j)^T M(B) (\bm x_i-\bm x_j)} \\
&+ \sum_s \gamma_s \sqrt{(\bm x_i-\bm x_j)^T M(A_s) (\bm x_i-\bm x_j)}.
\end{split}
\end{equation}
The calculation of $\gamma_b$ in (\ref{equ_distance}) is slightly different from that in (\ref{eq:D-nonoverlap}). In the following sections, we use an approximation of $\gamma_b$ for simplicity:  $\gamma_b = \max(1- \sum_s \gamma_s,0)$.

\section{Classifier and Learnability}

Lipschitz continuous functions are a family of smooth functions which are learnable \cite{gottlieb2014efficient}. In this paper, we select Lipschitz continuous functions as the classifiers. Based on the resultant learning bounds, we obtain the terms to regularize in order to improve the generalization ability.

\subsection{Classifier}

In the Euclidean space, it is intuitive to see the following classifier gives the same classification results as 1-NN:
\begin{equation*}
    f(\bm x)= \min D_{set}(\bm x, \bm X^-)-\min D_{set}(\bm x, \bm X^+),
\end{equation*}
where $f(\bm x)<0$ indicates that $\bm x$ belongs to negative class and $f(\bm x)>0$ indicates that $\bm x$ belongs to positive class; $ D_{set}(\bm x, \bm X^{-/+}) = \{D(\bm x, \bm x_t)|\forall x_t \in \mbox{negative class / positive class}\}$ is the set that contains the Euclidean distance values between $\bm x$ and any instance of the negative or positive class, and $D(\bm x_i, \bm x_j)$ indicates the Euclidean distance between $\bm x_i$ and $\bm x_j$.

K-NN considers more nearby instances and hence is more robust than 1-NN. A similar extension to consider more nearby instances based on the above equation is as follows:
\begin{equation}\label{equ_classifier}
\begin{split}
        f(\bm x)= &\frac{1}{K} \sumkmin  D_{set}(\bm x, \bm X^-)- \\
        & \frac{1}{K} \sumkmin D_{set}(\bm x, \bm X^+),
\end{split}
\end{equation}
where $\sumkmin$ denotes the sum of the $K$ minimal elements of the set. This function will be used as the classifier in our algorithm.

\subsection{Learnability of the Classifier with Local Metrics}

We will discuss the learnability of functions based on the Lipschitz constant, which characterizes the smoothness of a function. The smaller the value of Lipschitz constant, the more smooth the function is.
\begin{defn}
(\cite{weaver1999lipschitz})  The \textit{Lipschitz constant} of a function $f$ is
\begin{equation*}
\begin{split}
  \Lip(f) & = \inf\{ C\in \mathds{R}| \forall \bm x_i, \bm x_j \in \mathcal X, \\
  & \rho_{\mathcal Y} (f(\bm x_i), f(\bm x_j)) \leq C \rho_{\mathcal X}(\bm x_i, \bm x_j) \}  \\
    & =  \sup_{\bm x_i, \bm x_j \in \mathcal X: \bm x_i \neq \bm x_j} \frac{\rho_{\mathcal Y} (f(\bm x_i), f(\bm x_j))}{\rho_{\mathcal X}(\bm x_i, \bm x_j)}.
\end{split}
\end{equation*}
\end{defn}


\begin{pro}
\label{pro_lip_comp}
(\cite{weaver1999lipschitz}) Let $\Lip(f)\leq L_f$ and $\Lip(g) \leq L_g$ , then\\
(a) $\Lip(f+g) \leq L_f+L_g$; \\
(b) $\Lip(f-g) \leq L_f+L_g$; \\
(c) $\Lip(a f) \leq |a| L_f$, where $a$ is a constant.
\end{pro}

\begin{pro}\label{pro_lip_sumkmin}
Let the Lipschitz constant of $f_k(\bm x) \leq L_k, k=1,\dots, K$, then the 
Lipschitz constant of $ \sumkmin \{f_k(\bm x), k=1,\dots, K\}$ is bounded by $K\max_k L_k$.
\end{pro}
\begin{proof}
$\forall x_i, x_j \in \mathcal X, k\in\{1,\dots, K\}$
\begin{equation*}
\begin{split}
   & \sumkmin\{f_k(\bm x_i)\}  \\
   = & \sumkmin\{f_k(\bm x_j + (\bm x_i - \bm x_j))\}  \\
   \leq &  \sumkmin\{f_k(\bm x_j) + L_k\|\bm x_i - \bm x_j\|\}  \\
   \leq &  \sumkmin\{f_k(\bm x_j) + (\max_k L_k) \|\bm x_i - \bm x_j\|\}\\
   = & \sumkmin\{f_k(\bm x_j)\}+  K (\max_k L_k) \|\bm x_i - \bm x_j\|.
\end{split}
\end{equation*}
Therefore,
\begin{equation*}
\begin{split}
    &\sumkmin\{f_k(\bm x_i)\}-  \sumkmin\{f_k(\bm x_j)\} \\
    &\leq  K (\max_k L_k) \|\bm x_i - \bm x_j\|.
\end{split}
\end{equation*}
Based on the definition of Lipschitz constant, the proposition is proved.
\end{proof}

\begin{lemma}
\label{lipschitz}
With distance defined with (\ref{equ_distance}), the Lipschitz constant of the classifier illustrated by (\ref{equ_classifier}) is bound by $2 (\sum_s \|M(A_s)\|_F + \|M(B)\|_F)$, where  $\|\bm \cdot \|_F$ denotes the matrix Frobenius norm.  
\end{lemma}
\begin{proof}
Let $d_M(\bm x, \bm x_k)$ denote the Mahalanobis distance with metric $M$:
\begin{equation*}
    d_M(\bm x, \bm x_k)=\sqrt{(\bm x-\bm x_k)^T M (\bm x-\bm x_k)}.
\end{equation*}
With the identity matrix $I$, $d_I(\bm x, \bm x_k)$ is the Euclidean distance.

The Lipschitz constant of $f_1(\bm x)=d_M(\bm x, \bm x_k)$ is bounded by $\|M\|_F$ as follows:
\begin{equation*}
\begin{split}
    \Lip(f_1)&= \frac{f_1(\bm x)-f_1(\bm x_k)}{d_I(\bm x, \bm x_k)}\\
    &=  \frac{ d_M(\bm x, \bm x_k)- d_M(\bm x, \bm x_k)}{d_I(\bm x, \bm x_k)}
    \\
     & \leq \frac{d_M(\bm x, \bm x_k)}{d_I(\bm x, \bm x_k)}\\
     & \leq \frac{d_I(\bm x, \bm x_k)\|M\|_F}{d_I(\bm x, \bm x_k)}\\
     & = \|M\|_F,
\end{split}
\end{equation*}
where the first inequality follows the triangle inequality of distance, and the second inequality is based on the fact that matrix Frobenius norm is consistent with the vector $l_2$ norm\footnote{The consistence between a matrix norm $\|\cdot\|_M $ and a vector norm $\|\cdot\|_v$ indicates $\|\bm A \bm b\|_v \leq \|\bm A\|_M \|\bm b\|_v$, where $\bm A$ is a matrix, $\bm v$ is a vector, $\|\cdot\|_M$ is a matrix norm and $\|\cdot\|_v$ is a vector norm.}, i.e.
\begin{equation*}
\|(\bm x - \bm x_k)^T \bm M (\bm x - \bm x_k) \|_2 
          \leq  \|\bm x - \bm x_k\|_2^2 \|\bm M\|_F\\.
\end{equation*}

According to the definition of distance in (\ref{equ_distance}), we have
\begin{equation*}
    \begin{split}
    D_{M}(\bm x, \bm x_k)
    & \leq \sum_s  D_{M(A_s)}(\bm x, \bm x_k) + D_{M(B)}(\bm x, \bm x_k);
    \end{split}
\end{equation*}
and it follows Proposition~\ref{pro_lip_comp} that
\begin{equation*}
    \Lip( D_{M}(\bm x, \bm x_k)) \leq \sum_s \|M(A_s)\|_F + \|M(B)\|.
\end{equation*}
Based on the Lipschitz constant of $D_M(\bm x, \bm x_k)$ and the composition property illustrated  Proposition \ref{pro_lip_sumkmin}, 
\begin{equation*}
\begin{split}
   & \Lip( \sumkmin \{ D_{M}(\bm x, \bm x_k), k=1, \dots K \}) \\
    \leq &
    K \left\{ \sum_s \|M(A_s)\|_F + \|M(B)\| \right\}.
\end{split}
\end{equation*}
Finally, based on  Proposition~\ref{pro_lip_comp}, 
$f(\bm x)$ in (\ref{equ_classifier}) is bounded by $2(\sum_s \|M(A_s)\|_F + \|M(B)\|)$. 
\end{proof}

Combining the results of Proposition \ref{lipschitz} and the Corollary 6 of \cite{gottlieb2014efficient}, we can obtain the following Corollary.

\begin{cor}
Let metric space $(\mathcal X, \rho)$ have doubling dimension $\ddim(\mathcal X)$ and let $\mathcal F$ be the collection of real valued functions over $\mathcal X$ with the Lipschitz constant at most $L$. Then for any $f \in \mathcal F$ that classifies a sample of size $n$ correctly, 
if $f$ is correct on all but $k$ examples, we have with probability at least $1-\delta$
\begin{equation}
\label{Lipschitz_bound}
\begin{split}
    &P\{(\bm x, t): \sign [f(\bm x)] \neq t\} \\ 
    &\leq \frac{k}{n}+ \sqrt{ \frac{2}{n} (c\log_2 (34 en/c) \log_2 (578n) + \log_2 (4/\delta))},
\end{split}
\end{equation}
where
\begin{equation*}
    c= \Big(16(\sum_s \|M(A_s)\|_F + \|M(B)\|_F) \diam(\mathcal X, \rho)\Big)^{\ddim(\mathcal X)+1}.
\end{equation*}
$\diam$ denotes the diameter of the space and $\ddim$ denotes doubling dimension\footnote{The detailed definition can be found in~\cite{gottlieb2014efficient}}.
\end{cor}
The above learning bound illustrates the generalization ability, i.e. the difference between the expected error $P\{(\bm x, t): \sign [f(\bm x)] \neq t\}$ and the empirical error $k/n$.
Based on the bound, reducing the value of $\sum_s \|M(A_s)\|_F + \|M(B)\|_F$ would help reduce the gap between the empirical error and the expected error.
In other words, the learning bound indicates that regularizing $\sum_s \|M(A_s)\|_F + \|M(B)\|_F$ would help improve the generalization ability of the classifier.

\section{Optimization Problem}

\begin{table*}
\caption{Partial gradients of $\frac{\partial \gamma}{\partial \bm o}$ and $\frac{\partial \gamma}{\partial r}$ in different cases. }\label{gradient}
\begin{center}
\begin{tabular}{c|c|c|c}
\hline
$\Delta$ & $\lambda_u,\lambda_v$ & $\gamma$ & gradient\\
\hline
$\Delta \leq 0$ &     &       &     
$\frac{\partial \gamma}{\partial \bm o}=\bm 0,
\frac{\partial \gamma}{\partial r}=0$\\
\hline
            &   $\lambda_u<0, \lambda_v < 0$ & 0   &\\
$\Delta > 0$&  $\lambda_u<0, \lambda_v > 1$  & 1   &
$\frac{\partial \gamma}{\partial \bm o}=\bm 0,
\frac{\partial \gamma}{\partial r}=0$\\
           &  $\lambda_u>1,  \lambda_v  > 1$& 0    &\\
\hline
$\Delta > 0$ & $0\leq \lambda_u\leq 1, 0\leq \lambda_v  \leq 1$ & $\lambda_v-\lambda_u$& 
$
\frac{\partial \gamma }{\partial \bm o} =  4 \Delta^{-1/2}  [\bm x_i + \frac{-b}{2a}(\bm x_j -\bm x_i) - \bm o]
$\\
&      &     & 
$
\frac{\partial \gamma}{\partial r}= 4 \Delta^{-1/2} r 
$\\
\hline
$\Delta > 0$ &  $\lambda_u<0, 0\leq \lambda_v  \leq1$ & $\lambda_v$ &
$\frac{\partial \gamma }{\partial \bm o}
=   \frac{1}{2a} 
\big[2(\bm x_j - \bm x_i) +  \frac{1}{2} \Delta^{-1/2}\big(-4 b (\bm x_j - \bm x_i) -8a(\bm o - \bm x_i) \big) \big]$\\
& & &
$\frac{\partial \gamma }{\partial r}
= 2\Delta^{-1/2}r $\\
\hline
$\Delta > 0$ &  $0\leq \lambda_u\leq 1,  \lambda_v  > 1$ & $1-\lambda_u$ & 
$
\frac{\partial \gamma}{\partial \bm o}=
 \frac{1}{2a} \big[-2(\bm x_j - \bm x_i) +  \frac{1}{2} \Delta^{-1/2}\big(-4 b (\bm x_j - \bm x_i) -8a(\bm o - \bm x_i) \big) \big]
$\\
& & &
$\frac{\partial \gamma}{\partial r}
= 2\Delta^{-1/2}r $\\
\hline
\end{tabular}
\end{center}
\end{table*}

\subsection{Objective Function}

Based on the discussion in previous sections, with hinge loss and the regularization terms of $\sum_s \|M(A_s)\|_F + \|M(B)\|_F$, we propose the following optimization problem:
\begin{equation}
\begin{array}{cc}
\min\limits_{\Theta, \bm \xi}&  \frac{1}{N_1} \sum_{(i,j)}  \xi_{ij} +  \frac{1}{N_2}
\sum_{(m,n)} \xi_{mn} + \alpha \|\bm M(B)\|_F \\
&+\alpha \sum_s\|\bm M(A_s)\|_F \\ 
s.t. &  D_M (\bm x_i, \bm x_j) \leq 1-C+\xi_{ij} \\
&  D_M (\bm x_m, \bm x_n) \geq 1+C+\xi_{mn} \\
& \xi_{ij}, \xi_{mn} \geq 0, \bm M \in \bm M_{+}\\
&  i= 1,\dots, N, j\rightarrow i, m \nrightarrow n,
\end{array}
\end{equation}
where $\Theta=\{M(A_s), M(B), \bm o, \bm r\}$ denotes the set of parameters to be optimized; $j\rightarrow i$ indicates that $\bm x_j$ is $\bm x_i$'s $K$ nearest neighbor comparing against all instances in the same class; $m \nrightarrow n$ indicates that $\bm x_m$ is $\bm x_n$'s $K$ nearest neighbor comparing against all instances in the different class; and $\xi_{ij}$ and $\xi_{mn}$ indicates the errors. 
The regularization terms of $\|\bm M(B)\|_F$ and $\sum_s\|\bm M(A_s)\|_F$ control the complexity of metrics; 
$\alpha$ is a trade-off parameters; and $C$ is a constant which has the intuition of margin. 

The parameters to be optimized include local metrics $M(A_s)$, background metric $M(B)$, centers of influential regions $\bm o_s$ and radius of influential regions $r_s$. Thus in the proposed algorithm, we will learn the locations of influential regions ($\bm o_s, r_s$) and the metrics of influential/background regions ($M(B),M(A_s)$) under a same framework.

\subsection{Gradient Descent}

With $D_{M(A_s)}$ and $D_{M(B)}$ being the Mahalanobis distances, the optimization problem is not a convex problem even when we fix $\bm o, \bm r$ and update $M(A_s)$ and $M(B)$. Thus we simply adopt the gradient descent algorithm:
\begin{equation*}
    \Theta^{t+1} = \Theta^{t} - \beta \frac{\partial g} {\partial \Theta}|_{\Theta^{t}},
\end{equation*}
where $\beta$ is the learning rate, and the superscript $t$ denotes the time step during optimization. 

The objective function $g$ is
\begin{equation*}
\begin{split}
    &g= \\
    &\frac{1}{N_2} [1+C-D_M(\bm x_m, \bm x_n)]_+ +\frac{1}{N_1} [D_M(\bm x_i, \bm x_j)-(1-C)]_+\\
    &+ \alpha \sum_s \|M(A_s)\|_F + \alpha \|M(B)\|_F,
\end{split}
\end{equation*}
where the distance is
\begin{eqnarray*}
 D_{M}(\bm x_i, \bm x_j)&= & [1-\sum_s \gamma_s(\bm o_s, r_s)]_+  D_{M(B)} (\bm x_i,\bm x_j) +\\
 && \sum_s \gamma_s (\bm o_s, r_s)  D_{M(A_s)} (\bm x_i,\bm x_j).
\end{eqnarray*}
Here, $\gamma_s$ is written as $\gamma_s (\bm o_s, r_s)$ to remind us that $\gamma_s $ is a function of the location parameters $\bm o_s$ and $r_s$; $[x]_+=\max(x,0)$.

The gradient with respect to each set of parameters is
\begin{equation*}
\begin{split}
\frac{\partial g} {\partial \Theta}& |_{\Theta^{t}} = \\
& \frac{1}{N_2}  {\sum_{(m,n)}\mathbf 1[1+C-D_{M^t}(\bm x_m, \bm x_n) >0]}
\frac{\partial D_{M}(\bm x_m, \bm x_n)}{\partial \Theta}|_{\Theta^t}\\
+ & \frac{1}{N_1}  \sum_{(i,j)}\mathbf 1[D_{M^t}(\bm x_i, \bm x_j) -(1-C) >0]  \frac{\partial D_M(\bm x_i, \bm x_j)}{\partial \Theta}|_{\Theta^t}.
\end{split}
\end{equation*}
If the gradient is with respect to $M(B)$ and $M(A^s)$, then another shrinkage term of $\frac{\alpha M(B)}{2\|M(B)\|} $ or $\frac{\alpha M(A_s)}{2\|M(A_s)\|} $ from the Frobenius norm regularization term needs to be added into the above formula.

Now we will discuss $\frac{\partial D_M(\bm x_i, \bm x_j)}{\partial \Theta}|_{\Theta^t}$ for the parameters $M(A)$, $M(B)$, $\bm o^s$, $\bm r^s$ separately:
\begin{equation*}
\begin{split}
  & \frac{\partial D(\bm x_i, \bm x_j)}{\partial M(B)}|_{\Theta^t}   = \frac{ \mathbf 1[\gamma_b(\bm o_s^t, r_s^t) >0 ] }{2} \gamma_b(\bm o_s^t, r_s^t)\times\\
  &[(\bm x_i - \bm x_j)^T M^t(B)(\bm x_i - \bm x_j) ]^{-1/2} (\bm x_i - \bm x_j)(\bm x_i - \bm x_j)^T,
\end{split}
\end{equation*}
where $\gamma_b(\bm o_s^t, r_s^t)=1-\sum_s \gamma_s (\bm o_s^t, r_s^t)$;
\begin{equation*}
\begin{split}
   &\frac{\partial D(\bm x_i, \bm x_j)}{\partial M(A_s)}|_{\Theta^t} = \frac{ \gamma_s(\bm o_s^t, r_s^t) }{2}\times\\
   &[(\bm x_i - \bm x_j)^T M^t(A_s)(\bm x_i - \bm x_j) ]^{-1/2} (\bm x_i - \bm x_j)(\bm x_i - \bm x_j)^T;
\end{split}
\end{equation*}
\begin{equation*}
\begin{split}
   \frac{\partial D(\bm x_i, \bm x_j)}{\partial \bm o_s}|_{\Theta^t} =& \mathbf 1[1-\sum_s \gamma_s (\bm o_s^t, r_s^t) >0] D_{M^t(B)} \frac{\partial \gamma_s}{\partial \bm o_s} +\\
   & D_{M^t(A_s)}(\bm x_i, \bm x_j) \frac{\partial \gamma_s}{\partial \bm o_s},
\end{split}
\end{equation*}
where $\frac{\partial \gamma}{\partial \bm o}$ could be obtained as illustrated in Table~\ref{gradient};
\begin{equation*}
\begin{split}
   \frac{\partial D(\bm x_i, \bm x_j)}{\partial r_s}|_{\Theta^t} =& \mathbf 1[1-\sum_s \gamma_s (\bm o_s^t, r_s^t) >0] D_{M^t(B)} \frac{\partial \gamma_s}{\partial r_s} +\\
   & D_{M^t(A_s)}(\bm x_i, \bm x_j) \frac{\partial \gamma_s}{\partial r_s},
\end{split}
\end{equation*}
where $\frac{\partial \gamma}{\partial r}$ could be obtained as illustrated in Table~\ref{gradient}.

In this way, all of the gradients with respect to each set of parameters could be obtained and we can then use gradient descent to solve the optimization problem.

Initial values are very important for non-convex optimization problems. 
We adopt a heuristic method to initialize the parameters as follows. 1) Extract local discriminative direction $h(\bm x)\in R^F$ for each training instance $\bm x$, where $F$ indicates the number of features of $\bm x$:
\begin{equation*}
    h(\bm x_i)[f] = \sum_{k \nrightarrow i}|\bm x_k[f]-\bm x_i[f]| - \sum_{j\rightarrow i} |\bm x_j[f]-\bm x_i[f]|,
\end{equation*}
where $\bm x[f]$ indicates the $f$th dimension of vector $\bm x$;
$j\rightarrow i$ indicates $\bm x_j$ is $\bm x_i$'s $K$ nearest neighbor comparing against all instances in the same class; $k \nrightarrow i$ indicates $\bm x_k$ is $\bm x_i$'s $K$ nearest neighbor comparing against all instances in the different class.
2) Cluster with augmented features: $[\bm x, h(\bm x)]$ are used to cluster the instances into $S$ clusters.
3) Initialize the parameters:
Cluster centers are initialized as $\bm o_s$;
the distance between $80$ percentiles and the cluster center is set as initial value of $r_s$;
the local metric is set as $M(A_s) = I + 0.1\times {\rm diag}({\rm mean}( h(\bm x), {\bm x}\in \mbox{cluster } s))$,
where ${\rm diag}$ is an operation which returns a square diagonal matrix with elements of the input vector on the main diagonal. 


\section{Experiments}\label{sec:Experiments}

\begin{table}
\caption{Characteristics of 14 data sets: The total number of instances (and the numbers of instances in each class in brackets) and the number of features.}\label{dataset}
\begin{center}
\begin{tabular}{l| r| r}
\hline
 & Instances  & Features \\
\hline 
Australian & 690 (383, 307) & 14  \\
\hline   		
Breastcancer & 683 (444, 239) & 10  \\
\hline
Diabetes & 768 (268, 500) &8  \\
\hline
Fourclass & 862(555, 307) & 2\\
\hline
German & 1000 (700, 300) & 24 \\
\hline
Haberman & 206(81, 125) & 3\\
\hline  
Heart & 270 (150, 120) & 13\\
\hline
ILPD & 583(167, 416) & 10 \\
\hline
Liverdisorders& 345(145, 200) & 6\\
\hline  
Monk1 & 556 (278, 278) & 6\\
\hline
Pima&768(268, 500) &8\\
\hline
Planning & 182 (52, 130) & 12\\
\hline
Vote & 435 (168, 267) & 16  \\
\hline
WDBC & 569 (357, 212)  & 30\\
\hline
\end{tabular}
\end{center}
\end{table}

\begin{table*}
\caption{Metric learning algorithm Results: Mean accuracy and std are reported with the best ones in bold; `$\#$ of best' indicates the number of data sets that an algorithm performs the best.}\label{Experiment_results} 
\centering
\small
\resizebox{\textwidth}{!}{
\begin{tabular}{c|c|c|c|c|c|c|c|c|c}
\hline
Datasets        & LMNN       & ITML       & NCA        & MCML       & GMML                & RVML       & SCML                & R2LML               & Our                 \\ \hline
Australian     & 78.80$\pm$2.57 & 77.17$\pm$1.94 & 79.96$\pm$1.63 & 78.77$\pm$1.70 & 84.35$\pm$1.04          & 83.01$\pm$1.58 & 82.25$\pm$1.40          & 84.67$\pm$1.32          & \textbf{84.78$\pm$1.42} \\ \hline
Breastcancer   & 95.91$\pm$0.69 & 96.39$\pm$1.04 & 95.00$\pm$1.52 & 96.35$\pm$0.77 & \textbf{97.26$\pm$0.81} & 95.77$\pm$1.09 & 97.01$\pm$0.91          & 97.01$\pm$0.66          & 97.15$\pm$0.97          \\ \hline
Diabetes       & 69.16$\pm$1.44 & 69.09$\pm$1.24 & 68.47$\pm$2.46 & 69.19$\pm$1.18 & 74.16$\pm$2.58          & 71.04$\pm$2.60 & 71.49$\pm$2.21          & 73.80$\pm$1.37          & \textbf{75.19$\pm$1.47} \\ \hline
Fourclass      & 72.06$\pm$2.31 & 72.09$\pm$2.22 & 72.06$\pm$2.46 & 72.06$\pm$2.43 & 76.12$\pm$1.87          & 70.46$\pm$1.40 & 75.54$\pm$1.42          & 76.12$\pm$1.91          & \textbf{79.71$\pm$1.11} \\ \hline
German    & 67.85$\pm$1.54 & 66.95$\pm$2.05 & 69.95$\pm$2.88 & 67.67$\pm$1.48 & 71.55$\pm$1.12          & 71.65$\pm$1.78 & 70.90$\pm$2.65          & \textbf{72.90$\pm$1.83} & 72.45$\pm$1.41          \\ \hline
Haberman       & 67.89$\pm$3.34 & 67.97$\pm$4.05 & 67.40$\pm$3.33 & 67.56$\pm$2.75 & 71.22$\pm$3.35          & 66.67$\pm$2.30 & 69.19$\pm$2.47          & 71.06$\pm$3.39          & \textbf{74.07$\pm$3.97} \\ \hline
Heart          & 76.20$\pm$3.82 & 76.94$\pm$3.30 & 75.56$\pm$2.01 & 77.22$\pm$3.66 & 81.20$\pm$2.69          & 77.69$\pm$4.05 & 78.98$\pm$3.24          & \textbf{82.04$\pm$3.81} & 81.67$\pm$3.14          \\ \hline
ILPD           & 66.97$\pm$2.13 & 68.67$\pm$2.83 & 66.80$\pm$1.19 & 67.48$\pm$2.58 & 67.14$\pm$2.17          & 67.95$\pm$2.90 & 68.03$\pm$2.90          & 65.85$\pm$2.22          & \textbf{69.27$\pm$1.60} \\ \hline
Liverdisorders & 61.01$\pm$4.80 & 57.17$\pm$4.01 & 59.78$\pm$3.44 & 60.65$\pm$5.12 & 63.84$\pm$5.43          & 64.64$\pm$3.93 & 61.74$\pm$4.57          & \textbf{66.81$\pm$3.68} & 65.29$\pm$3.67          \\ \hline
Monk1          & 88.43$\pm$2.63 & 77.31$\pm$1.27 & 93.09$\pm$5.70 & 79.42$\pm$1.91 & 75.02$\pm$2.61          & 89.24$\pm$2.68 & \textbf{97.53$\pm$0.85} & 89.24$\pm$1.54          & 96.46$\pm$2.96          \\ \hline
Pima           & 68.54$\pm$1.64 & 67.95$\pm$2.01 & 65.91$\pm$3.04 & 68.31$\pm$2.33 & 72.95$\pm$1.84          & 69.45$\pm$1.68 & 71.14$\pm$2.64          & 72.34$\pm$1.54          & \textbf{74.32$\pm$1.27} \\ \hline
Planning       & 60.41$\pm$5.29 & 62.19$\pm$2.31 & 58.49$\pm$8.59 & 62.88$\pm$4.35 & 65.20$\pm$5.49          & 55.07$\pm$7.35 & 61.92$\pm$4.99          & 63.84$\pm$3.43          & \textbf{67.40$\pm$3.81} \\ \hline
Voting         & 94.83$\pm$0.77 & 90.75$\pm$1.44 & 94.77$\pm$0.92 & 92.64$\pm$1.58 & 95.17$\pm$1.88          & 95.75$\pm$1.26 & 95.00$\pm$1.30          & \textbf{96.32$\pm$1.19} & 95.75$\pm$1.30          \\ \hline
WDBC           & 96.58$\pm$1.12 & 94.91$\pm$0.92 & 96.58$\pm$0.85 & 95.70$\pm$0.90 & 96.71$\pm$0.78          & 96.58$\pm$1.34 & 96.97$\pm$0.89          & 96.93$\pm$1.67          & \textbf{97.28$\pm$0.92} \\ \hline
\# of best     & 0          & 0          & 0          & 0          & 1                   & 0          & 1                   & 4                   & 8                   \\ \hline 
\end{tabular}}
\end{table*}

We compare our algorithm with eight established metric learning algorithms from two categories: 1) The most cited algorithms, including Large Marin Nearest Neighbor (LMNN)~\cite{weinberger2009distance}, Information Theoretic metric learning (ITML)~\cite{davis2007information}, Neighborhood Component Analysis (NCA)~\cite{goldberger2005neighbourhood} and Metric learning by Collapsing Classes (MCML)~\cite{globerson2006metric}; (2) the most state-of-the-art algorithms, including Regressive Virtual Metric Learning (RVML)~\cite{perrot2015regressive}, Geometric Mean Metric Learning (GMML)~\cite{zadeh2016geometric}, Sparse Compositional Metric Learning (SCML)~\cite{shi2014sparse} and Reduced-Rank Local Distance Metric Learning (R2LML)~\cite{huang2013reduced}.
LMNN and ITML are implemented with metric-learn toolbox\footnote{https://all-umass.github.io/metric-learn/};
NCA and MCML are implemented with the drToolbox\footnote{https://lvdmaaten.github.io/drtoolbox/}; and GMML, RVML, SCML and R2LML are implemented by using the authors' code.

In our experiments, we focus on binary classification on 14 publicly available data sets from the websites of UCI\footnote{https://archive.ics.uci.edu/ml/datasets.html} and LibSVM\footnote{https://www.csie.ntu.edu.tw/~cjlin/libsvmtools/datasets/binary.html}, namely Australian, Breastcancer, Diabetes, Fourclass, Germannumber, Haberman, Heart, ILPD, Liverdisorders, Monk1, Pima, Planning, Voting and WDBC. 
The characteristics of these data sets are summarized in Table~\ref{dataset}. All data sets are pre-processed by firstly subtracting the mean and dividing by the standard deviation, and then normalizing the L2-norm of each instance to one.  

For each data set, $60\%$ instances are randomly selected as training samples and the rest for testing. 
This process is repeated 10 times and the mean accuracy and the standard deviation are reported.
We use 10-fold cross-validation to select the trade-off parameters in the compared algorithms, namely the regularization parameter of LMNN (from $\{0.1,0.3,0.5,0.7,0.9\}$), $\gamma$ in ITML (from $\{0.25,0.5,1,2,4\}$), $t$ in GMML (from $\{0.1,0.3,0.5,0.7,0.9\}$) and $\lambda$ in RVML (from $\{10^{-5},10^{-4},10^{-3},10^{-2},10^{-1},1,10\})$.  
All other parameters are set as default. 
For our algorithm, we set the parameters as follows: $\alpha$ and $C$ in the optimization formula are $0.1$ and $0.5$ respectively; $K$ in the classifier is $10$; and the number of clusters when initializing the parameters is $4$. 

As shown in Table~\ref{Experiment_results},  the proposed algorithm achieves the best accuracy on eight data sets out of the 14 data sets. None of the other algorithms performs the best in more than 4 data sets.  
In cases which our algorithm is not leading, it performs quite nice and stays close to the best one. 
Such encouraging results demonstrate the effectiveness of our proposed method.

\section{Conclusions and future work}

In this short paper, by introducing influential regions, we define a very intuitive distance and propose a novel local metric learning method. The distance can be computed efficiently and encouraging results are obtained on public data sets.

It is straightforward to extend the proposed algorithm to multi-class cases and adopt more advanced optimization techniques. Other metrics or other types of influential regions can also be adopted for specific tasks. Domain knowledge can be embedded into the partition of the regions. Tighter learning bounds and resultant penalty terms would be our future work.

\bibliographystyle{IEEEtran}
\bibliography{myref_new}

\begin{thebibliography}{10}
\providecommand{\url}[1]{#1}
\csname url@samestyle\endcsname
\providecommand{\newblock}{\relax}
\providecommand{\bibinfo}[2]{#2}
\providecommand{\BIBentrySTDinterwordspacing}{\spaceskip=0pt\relax}
\providecommand{\BIBentryALTinterwordstretchfactor}{4}
\providecommand{\BIBentryALTinterwordspacing}{\spaceskip=\fontdimen2\font plus
\BIBentryALTinterwordstretchfactor\fontdimen3\font minus
  \fontdimen4\font\relax}
\providecommand{\BIBforeignlanguage}[2]{{%
\expandafter\ifx\csname l@#1\endcsname\relax
\typeout{** WARNING: IEEEtran.bst: No hyphenation pattern has been}%
\typeout{** loaded for the language `#1'. Using the pattern for}%
\typeout{** the default language instead.}%
\else
\language=\csname l@#1\endcsname
\fi
#2}}
\providecommand{\BIBdecl}{\relax}
\BIBdecl

\bibitem{xing2002distance}
E.~P. Xing, M.~I. Jordan, S.~Russell, and A.~Y. Ng, ``Distance metric learning
  with application to clustering with side-information,'' in \emph{Advances in
  neural information processing systems}, 2002, pp. 505--512.

\bibitem{weinberger2009distance}
K.~Q. Weinberger and L.~K. Saul, ``Distance metric learning for large margin
  nearest neighbor classification,'' \emph{The Journal of Machine Learning
  Research}, vol.~10, pp. 207--244, 2009.

\bibitem{davis2007information}
J.~V. Davis, B.~Kulis, P.~Jain, S.~Sra, and I.~S. Dhillon,
  ``Information-theoretic metric learning,'' in \emph{Proceedings of the 24th
  international conference on Machine learning}.\hskip 1em plus 0.5em minus
  0.4em\relax ACM, 2007, pp. 209--216.

\bibitem{jin2009regularized}
R.~Jin, S.~Wang, and Y.~Zhou, ``Regularized distance metric learning: Theory
  and algorithm,'' in \emph{Advances in neural information processing systems},
  2009, pp. 862--870.

\bibitem{guo2014guaranteed}
Z.-C. Guo and Y.~Ying, ``Guaranteed classification via regularized similarity
  learning,'' \emph{Neural computation}, vol.~26, no.~3, pp. 497--522, 2014.

\bibitem{cao2016generalization}
Q.~Cao, Z.-C. Guo, and Y.~Ying, ``Generalization bounds for metric and
  similarity learning,'' \emph{Machine Learning}, vol. 102, no.~1, pp.
  115--132, 2016.

\bibitem{verma2015sample}
N.~Verma and K.~Branson, ``Sample complexity of learning mahalanobis distance
  metrics,'' in \emph{Advances in Neural Information Processing Systems}, 2015,
  pp. 2584--2592.

\bibitem{hu2014discriminative}
J.~Hu, J.~Lu, and Y.-P. Tan, ``Discriminative deep metric learning for face
  verification in the wild,'' in \emph{Proceedings of the IEEE Conference on
  Computer Vision and Pattern Recognition}, 2014, pp. 1875--1882.

\bibitem{lu2015multi}
J.~Lu, G.~Wang, W.~Deng, P.~Moulin, and J.~Zhou, ``Multi-manifold deep metric
  learning for image set classification,'' in \emph{Proceedings of the IEEE
  Conference on Computer Vision and Pattern Recognition}, 2015, pp. 1137--1145.

\bibitem{liu2018metric}
W.~Liu, D.~Xu, I.~Tsang, and W.~Zhang, ``Metric learning for multi-output
  tasks,'' \emph{IEEE Transactions on Pattern Analysis and Machine
  Intelligence}, 2018.

\bibitem{hu2017sharable}
J.~Hu, J.~Lu, and Y.-P. Tan, ``Sharable and individual multi-view metric
  learning,'' \emph{IEEE transactions on pattern analysis and machine
  intelligence}, 2017.

\bibitem{yang2010boosting}
L.~Yang, R.~Jin, L.~Mummert, R.~Sukthankar, A.~Goode, B.~Zheng, S.~C. Hoi, and
  M.~Satyanarayanan, ``A boosting framework for visuality-preserving distance
  metric learning and its application to medical image retrieval,'' \emph{IEEE
  Transactions on Pattern Analysis and Machine Intelligence}, vol.~32, no.~1,
  pp. 30--44, 2010.

\bibitem{lu2014neighborhood}
J.~Lu, X.~Zhou, Y.-P. Tan, Y.~Shang, and J.~Zhou, ``Neighborhood repulsed
  metric learning for kinship verification,'' \emph{IEEE transactions on
  pattern analysis and machine intelligence}, vol.~36, no.~2, pp. 331--345,
  2014.

\bibitem{yan2014discriminative}
H.~Yan, J.~Lu, W.~Deng, and X.~Zhou, ``Discriminative multimetric learning for
  kinship verification,'' \emph{IEEE Transactions on Information forensics and
  security}, vol.~9, no.~7, pp. 1169--1178, 2014.

\bibitem{huang2017cross}
Z.~Huang, R.~Wang, S.~Shan, L.~Van~Gool, and X.~Chen, ``Cross
  {Euclidean}-to-{Riemannian} metric learning with application to face
  recognition from video,'' \emph{IEEE Transactions on Pattern Analysis and
  Machine Intelligence}, 2017.

\bibitem{wang2017tracklet}
B.~Wang, G.~Wang, K.~L. Chan, and L.~Wang, ``Tracklet association by online
  target-specific metric learning and coherent dynamics estimation,''
  \emph{IEEE transactions on pattern analysis and machine intelligence},
  vol.~39, no.~3, pp. 589--602, 2017.

\bibitem{frome2007learning}
A.~Frome, Y.~Singer, F.~Sha, and J.~Malik, ``Learning globally-consistent local
  distance functions for shape-based image retrieval and classification,'' in
  \emph{Computer Vision, 2007. ICCV 2007. IEEE 11th International Conference
  on}.\hskip 1em plus 0.5em minus 0.4em\relax IEEE, 2007, pp. 1--8.

\bibitem{wang2012parametric}
J.~Wang, A.~Kalousis, and A.~Woznica, ``Parametric local metric learning for
  nearest neighbor classification,'' in \emph{Advances in Neural Information
  Processing Systems}, 2012, pp. 1601--1609.

\bibitem{huang2013reduced}
Y.~Huang, C.~Li, M.~Georgiopoulos, and G.~C. Anagnostopoulos, ``Reduced-rank
  local distance metric learning,'' in \emph{Joint European Conference on
  Machine Learning and Knowledge Discovery in Databases}.\hskip 1em plus 0.5em
  minus 0.4em\relax Springer, 2013, pp. 224--239.

\bibitem{bohne2014large}
J.~Bohn{\'e}, Y.~Ying, S.~Gentric, and M.~Pontil, ``Large margin local metric
  learning,'' in \emph{European Conference on Computer Vision}.\hskip 1em plus
  0.5em minus 0.4em\relax Springer, 2014, pp. 679--694.

\bibitem{shi2014sparse}
Y.~Shi, A.~Bellet, and F.~Sha, ``Sparse compositional metric learning,'' in
  \emph{AAAI}, 2014, pp. 2078--2084.

\bibitem{saxena2015coordinated}
S.~Saxena and J.~Verbeek, ``Coordinated local metric learning,'' in
  \emph{Proceedings of the IEEE International Conference on Computer Vision
  Workshops}, 2015, pp. 127--135.

\bibitem{st2017sparse}
J.~St~Amand and J.~Huan, ``Sparse compositional local metric learning,'' in
  \emph{Proceedings of the 23rd ACM SIGKDD International Conference on
  Knowledge Discovery and Data Mining}.\hskip 1em plus 0.5em minus 0.4em\relax
  ACM, 2017, pp. 1097--1104.

\bibitem{noh2018generative}
Y.~Noh, B.~Zhang, and D.~Lee, ``Generative local metric learning for nearest
  neighbor classification.'' \emph{IEEE transactions on pattern analysis and
  machine intelligence}, vol.~40, no.~1, p. 106, 2018.

\bibitem{gottlieb2014efficient}
L.-A. Gottlieb, A.~Kontorovich, and R.~Krauthgamer, ``Efficient classification
  for metric data,'' \emph{Information Theory, IEEE Transactions on}, vol.~60,
  no.~9, pp. 5750--5759, 2014.

\bibitem{weaver1999lipschitz}
N.~Weaver and N.~Weaver, \emph{Lipschitz algebras}.\hskip 1em plus 0.5em minus
  0.4em\relax World Scientific, 1999.

\bibitem{goldberger2005neighbourhood}
J.~Goldberger, G.~E. Hinton, S.~T. Roweis, and R.~R. Salakhutdinov,
  ``Neighbourhood components analysis,'' in \emph{Advances in neural
  information processing systems}, 2005, pp. 513--520.

\bibitem{globerson2006metric}
A.~Globerson and S.~T. Roweis, ``Metric learning by collapsing classes,'' in
  \emph{Advances in neural information processing systems}, 2006, pp. 451--458.

\bibitem{perrot2015regressive}
M.~Perrot and A.~Habrard, ``Regressive virtual metric learning,'' in
  \emph{Advances in Neural Information Processing Systems}, 2015, pp.
  1810--1818.

\bibitem{zadeh2016geometric}
P.~Zadeh, R.~Hosseini, and S.~Sra, ``Geometric mean metric learning,'' in
  \emph{International Conference on Machine Learning}, 2016, pp. 2464--2471.

\end{thebibliography}

\end{document}